\documentclass[conference]{IEEEtran}
\IEEEoverridecommandlockouts
\usepackage{cite}
\usepackage{amsmath,amssymb,amsfonts}
\usepackage{algorithmic}
\usepackage{graphicx}
\usepackage{textcomp}
\usepackage{xcolor}

\usepackage{mathtools}
\usepackage{comment}
\usepackage{xspace}
\usepackage{enumitem}
\usepackage[ruled,linesnumbered,noend]{algorithm2e}
\usepackage{subfigure}
\usepackage{multirow}
\usepackage{booktabs}
\usepackage{url}
\usepackage{makecell}
\usepackage{amsthm,bm}
\usepackage{tablefootnote}
\usepackage{balance}
\usepackage{romannum}
\usepackage[hidelinks]{hyperref}

\def\BibTeX{{\rm B\kern-.05em{\sc i\kern-.025em b}\kern-.08em
    T\kern-.1667em\lower.7ex\hbox{E}\kern-.125emX}}    
\newcommand{\method}{\textsc{TensorCodec}\xspace}

\newcommand{\smallsection}[1]{\noindent\underline{\smash{\textbf{#1:}}}}

\newcommand{\nttd}{NTTD\xspace}
\newcommand{\nttD}{NTTD\xspace}

\newcommand{\ttc}{TT cores\xspace}
\newcommand{\ttd}{TTD\xspace}
\newcommand{\order}{\bm{\pi}}

\newcommand{\model}{\theta}
\newcommand{\vect}[1]{\mathbf{#1}}
\newcommand{\mat}[1]{\mathbf{#1}}
\newcommand{\tensor}[1]{\mathbf{\mathcal{#1}}}
\newcommand{\fold}[1]{#1^{\texttt{folded}}}
\newcommand{\reorder}[1]{{#1}_{\order}}

\newcommand{\fnorm}[1]{\lVert #1 \rVert_{F}}

\newcommand{\shape}{N_1 \times \cdots \times N_d}
\newcommand{\slice}[2]{{\tensor{X}}^{(#1)}(#2)}
\newcommand{\cslice}[3]{#1^{(#2)}(#3)}

\newcommand{\recon}[1]{\Tilde{#1}}
\newtheorem{problem}{Problem}
\newtheorem{theorem}{Theorem}

\newcommand{\floor}[1]{\left\lfloor #1 \right\rfloor}

\newcommand{\romanum}[1]{\MakeUppercase{\romannumeral #1}}

\definecolor{my_blue}{RGB}{47, 85, 151}
\definecolor{my_sky}{RGB}{0, 176, 240}
\definecolor{my_green}{RGB}{84, 130, 53}
\definecolor{darkgreen}{RGB}{0, 150, 0}

\definecolor{my_yellow}{RGB}{191, 144, 0}
\definecolor{my_violet}{RGB}{112, 48, 160}

\newcommand{\jihoon}[1]{\textcolor{darkgreen}{#1}}

\newcommand{\red}[1]{\textcolor{red}{#1}}

\SetKwComment{Comment}{$\triangleright$\ }{}

\SetCommentSty{mycommfont}

\begin{document}

\title{TensorCodec: Compact Lossy Compression of Tensors without Strong Data Assumptions \vspace{-2mm}}

\author{\IEEEauthorblockN{Taehyung Kwon\textsuperscript{1}, Jihoon Ko\textsuperscript{1}, Jinhong Jung\textsuperscript{2}, and Kijung Shin\textsuperscript{1}}
\IEEEauthorblockA{\textsuperscript{1}Kim Jaechul Graduate School of AI, KAIST, \textsuperscript{2}School of Software, Soongsil University}
taehyung.kwon@kaist.ac.kr, jihoonko@kaist.ac.kr, jinhong@ssu.ac.kr, kijungs@kaist.ac.kr
\vspace{-4.8mm}}
\maketitle

\begin{abstract}
Many real-world datasets are represented as tensors, i.e., multi-dimensional arrays
of numerical values.
Storing them without compression often requires substantial space, which grows exponentially with the order.
While many tensor compression algorithms are available, many of them rely on strong data assumptions regarding its order, sparsity, rank, and smoothness.


In this work, we propose \method, a lossy compression algorithm for general tensors that do not necessarily adhere to strong input data assumptions.
\method incorporates three key ideas. The first idea is Neural Tensor-Train Decomposition (\nttd)  where we integrate a recurrent neural network into 
Tensor-Train Decomposition to enhance its expressive power and alleviate the limitations imposed by the low-rank assumption.
Another idea is to fold the input tensor into a higher-order tensor to reduce the space required by \nttd. 
Finally, the mode indices of the input tensor are reordered to reveal patterns that can be exploited by \nttd for improved approximation.

Our analysis and experiments on $8$ real-world datasets demonstrate that \method is
 \textbf{(a) Concise:} it gives up to $7.38\times$ more compact compression than the best competitor with similar reconstruction error,
 \textbf{(b) Accurate:} given the same budget for compressed size, it yields up to $3.33\times$ more accurate reconstruction than the best competitor,
 \textbf{(c) Scalable:} its empirical compression time is linear in the number of tensor entries, and it reconstructs each entry in logarithmic time.
Our code and datasets are available at \url{https://github.com/kbrother/TensorCodec}.

\end{abstract}


\section{Introduction} \label{sec:intro}


A tensor is a multi-dimensional array of numerical values~\cite{kolda2009tensor} and can be considered a higher-order generalization of a matrix. Various real-world datasets, including sensory data~\cite{jang2020d}, stock market history~\cite{jang2021fast}, and extracted features from motion videos~\cite{wang2012mining}, are represented as tensors.


Many real-world tensors are large, making compression crucial for efficient storage. Storing a tensor in its uncompressed form requires space proportional to the total number of entries, which grows exponentially with the tensor's order. 
For instance, storing a tensor of size $704 \times 2049 \times 7997$ from the music dataset \cite{defferrard2016fma} as a double-precision floating-point value per entry consumes approximately $92$GB. 
This storage approach places a significant burden on memory-limited devices such as mobile and IoT devices. 
Moreover, transmitting such large datasets online
can lead to substantial communication costs.

Consequently, a variety of tensor compression methods have been devised, but many of them rely on specific assumptions regarding input data, particularly in terms of order, sparsity, rank, and smoothness.  For instance, numerous compression methods \cite{xu1998truncated,sun2007less} assume that the input tensor is of order two, i.e., a matrix. 
Additionally, many other compression methods \cite{smith2015splatt,kwon2023neukron} are tailored for compressing sparse tensors, which are tensors with a majority of zero entries. 
Another category of compression methods based on tensor decomposition \cite{hitchcock1927expression,tucker1966some,oseledets2011tensor,zhao2019learning} assumes that the input tensor has (approximately) a low-rank structure. 
Furthermore, some compression methods \cite{ballester2019tthresh,zhao2021optimizing,ma2019image,bhaskaran1997image} presume that tensor entries exhibit smooth variations, meaning that adjacent entries tend to be similar, as in images and videos.
However, many real-world tensors do not necessarily conform to these assumptions, as shown in Section~\ref{sec:exp:tradeoff}.

\textit{How can we compactly compress tensors  with minimal reconstruction error, without imposing strong assumptions on data properties?}
To address this question, we present \method, a lossy compression algorithm for general tensors that do not depend on strict input data assumptions.
To achieve this, we first introduce Neural Tensor-Train Decomposition (\nttd). In contrast to the original Tensor-Train Decomposition, where factor matrices are fixed for all entries of a mode index, we use a recurrent neural network to obtain these matrices, making them dependent on the other mode indices of the entries. This modification allows \nttd to effectively approximate input tensors with a limited number of parameters, even when they exhibit high rank.
Secondly, we fold the input tensor into a higher-order tensor, further reducing the number of parameters needed for \nttd. Lastly, we introduce a reordering algorithm for the mode indices of the input tensor, which uncovers patterns that \nttd can exploit for accurate approximation of the tensor entries.
The output of \method consists of (a) the parameters of the recurrent neural network in \nttd and (b) the mapping between the original and reordered mode indices, which are used to approximate the entries of the input tensor.

We demonstrate the advantages of \method through complexity analysis and comprehensive experiments on 8 real-world datasets, which are summarized as follows: 
\begin{itemize}[leftmargin=*]
    \item \textbf{Concise:} It gives up to 7.38$\times$ smaller output than the best-performing competitor with similar approximation error.
    \item \textbf{Accurate:} It outperforms the most accurate competitor in terms of approximation error by up to 3.33$\times$, while achieving comparable compression sizes.
    \item \textbf{Scalable:} Its empirical compression time scales linearly with the number of entries, while the reconstruction time scales logarithmically with the largest mode size.
\end{itemize}
We present related works in Section~\ref{sec:rel_works}, give preliminaries in Section~\ref{sec:prelims},
propose \method in Section~\ref{sec:expr}, review experiments in Section~\ref{sec:method}, and make conclusions in Section~\ref{sec:conclusion}.




\section{Related work} \label{sec:rel_works}
\vspace{-1mm}

\smallsection{Compression methods for low-rank tensors}
Tensor decomposition compresses tensors lossily into smaller matrices and tensors.
CP Decomposition (CPD)~\cite{carroll1970analysis} approximates the input tensor by the weighted sum of the outer products of columns of the same order in the matrices. That is, the outer product of the $i$-th columns of the matrices is computed for each $i$.
Tucker Decomposition (TKD)~\cite{tucker1966some} generalizes CPD by allowing outer products of columns of different orders, and TTHRESH~\cite{ballester2019tthresh} further compresses the outputs of TKD using run-length and arithmetic coding.
Tensor Train Decomppsition (TTD)~\cite{oseledets2011tensor} approximates each entry by the product of matrices that vary depending on mode indices, while
Tensor Ring Decomposition (TRD)~\cite{zhao2019learning} uses the trace of the product for approximation.
That is, if two entries share $k$ mode indices, $k$ matrices are shared for their approximation.
The approximation power of these methods is limited by their reliance on low-rank assumptions and linear operations, especially when the input tensor does not have a low rank and the parameter size 
is restricted, as shown in Section~\ref{sec:exp:tradeoff}.




\smallsection{Compression methods for smooth tensors}
Numerous tensor compression techniques assume smoothness in the input data, meaning that adjacent entries tend to have similar values. 
This assumption is especially prominent for image and video data. 
Compression of images and videos constitutes a distinct and advanced field of study \cite{bhaskaran1997image,ma2019image}, and consequently, this paper focuses on compression methods for tensors that are neither images nor videos.
Smoothness is also assumed when compressing scientific simulation data \cite{zhao2021optimizing}. 
A notable example, SZ3~\cite{zhao2021optimizing} performs interpolation for each entry, based on smoothness, and compactly but 
 lossily encodes the errors using Huffman coding. 
Naturally, the compression power of these methods diminishes when the smoothness assumption is not met, as shown empirically in Section~\ref{sec:exp:tradeoff}.

\smallsection{Compression methods for sparse tensors}
Many real-world tensors are sparse, i.e. the majority of their entries are zero. 
This characteristic is commonly leveraged in the aforementioned tensor decomposition methods for speed and memory efficiency \cite{bader2008efficient,kolda2008scalable,zhang2020fast} without affecting the compressed outputs. 
NeuKron~\cite{kwon2023neukron} optimizes the entries of a small tensor, referred to as a seed tensor, so that its generalized Kronecker power closely approximates the input tensor. To improve the fit, NeuKron first reorders the mode indices in the input tensor. NeuKron leverages the sparsity of the input tensor in three ways:
(a) it builds on the observation that real-world sparse matrices exhibit self-similarity and can thus be approximated by Kronecker powers~\cite{leskovec2007scalable}, (b) it reorders mode indices based on the sparsity patterns of the sub-tensors formed by the entries with each mode index, and (c) it enables rapid computation of the objective function by taking advantage of the sparsity.
Therefore, NeuKron has been applied only to extremely sparse tensors where the ratio of non-zero entries is less than 0.00354 \cite{kwon2023neukron}.
NeuKron and our proposed \method share similarities in that they both employ reordering and generalization using auto-regressive models. 
However, \method generalizes \ttd, instead of Kronecker powers, and reorders mode indices based on the entry values, not on sparsity patterns. It should be noticed that, unlike Kronecker powers, \ttd has been used to model various real-world tensors without being limited to sparse ones.
Due to these differences, \method significantly outperforms NeuKron in terms of both compressed size and approximation accuracy for tensors that are not extremely sparse, as demonstrated in Section~\ref{sec:exp:tradeoff}.

\smallsection{Usage of \ttd for neural networks}
\ttd~\cite{oseledets2011tensor} has been utilized to compress neural-network parameters in the form of tensors, including (a) weights of fully connected layers \cite{novikov2015tensorizing} and Recurrent Neural Networks \cite{yang2017tensor}, and (b) embeddings in recommender systems \cite{yin2021tt} and Graph Neural Networks \cite{yin2022nimble}.
These works employ \ttd for
compressing neural networks, which is distinctly different from
our approach of employing a neural network to enhance the
expressiveness of \ttd.



\section{preliminaries} \label{sec:prelims}
In this section, we introduce preliminaries, followed by a formal definition of the considered problem. Frequently-used notations are listed in Table~\ref{tab:symbol}.
For any positive integer $n$, $[n]:=\{0,\cdots,n-1\}$ denotes the set of integers from 0 to $n-1$.


\subsection{Basic Concepts and Notations}
\label{sec:prelim:concepts}

\smallsection{Matrix and tensor}
We denote matrices in boldface capital letters.
Given a real-valued matrix $\mathbf{M}$ of size $N_1 \times N_2$, the entry located in the $i$-th row and $j$-th column is denoted by $\mathbf{M}(i, j)$. 
Tensors are denoted by boldface Euler script letters.
The order of a tensor refers to the number of modes.
Let $\tensor{X}$ be a real-valued $d$-order tensor of size $N_1 \times \cdots \times N_d$.
The entry at the $(i_1, \cdots, i_d)$-th position of $\tensor{X}$ is denoted by $\tensor{X}(i_1, \cdots, i_d)$.

\smallsection{Slicing and reordering a tensor} \label{sec:prelim:concept:slice}
For a mode-$j$ index $i\in[N_j]$, $\cslice{\tensor{X}}{j}{i} \in \mathbb{R}^{N_1 \cdots  N_{j-1} \times N_{j+1}  \cdots  N_{d}}$ denotes the $i$-th slice of $\tensor{X}$ along the $j$-th mode, i.e., 
$\cslice{\tensor{X}}{j}{i} \coloneqq \tensor{X}(:_1, \cdots, :_{j-1}, i, :_{j+1}, \cdots, :_d)$, where $:_{k}$ indicates all possible mode-$k$ indices (i.e., those in $[N_k]$).
We consider reordering of mode indices. Let $\reorder{\tensor{X}}$ denote the tensor reordered from $\tensor{X}$ by a set $\order = \{\pi_1, \cdots, \pi_d\}$ of reordering functions, where each $\pi_i: [N_i] \rightarrow [N_i]$ is a bijective function from the set of the mode-$i$ indices to themselves. 
In $\reorder{\tensor{X}}$, the $(i_1, \cdots, i_d)$-th entry corresponds to the $(\pi_1(i_1), \pi_2(i_2), ..., \pi_d(i_d))$-th entry of $\tensor{X}$.

\smallsection{Frobenius norm}
The Frobenius norm $\fnorm{\tensor{X}}$ of $\tensor{X}$ is defined as the squared root of the squared sum of all  its entries, i.e.,
\begin{equation}
\fnorm{\tensor{X}} = \sqrt{\sum_{(i_1, \cdots, i_d) \in [N_1] \times \cdots \times [N_d]} \big(\tensor{X}(i_1, \cdots, i_d)\big)^{2} }.
\end{equation}

\subsection{Tensor-Train Decomposition (\ttd)}
\label{sec:prelim:ttd}
Tensor-Train Decomposition (\ttd)~\cite{oseledets2011tensor} decomposes a given $d$-order tensor $\tensor{X}$ into $d$ tensors $\tensor{G}_1$, $\cdots$, $\tensor{G}_d$, called \ttc, so that each entry of $\tensor{X}$ is approximated as follows:
\begin{equation}
    \mathcal{X}(i_1, \cdots, i_d) \approx \cslice{\tensor{G}_1}{2}{i_1} \cslice{\tensor{G}_2}{2}{i_2} \cdots \cslice{\tensor{G}_d}{2}{i_d},
\end{equation}
where, for all $k$, $\tensor{G}_k \in \mathbb{R}^{r_{k-1}\times N_k \times r_k}$, and $\cslice{\tensor{G}_k}{2}{i} \in \mathbb{R}^{r_{k-1} \times r_k}$ is the $i$-th slice of $\tensor{G}_k$ along the second mode.
%
%
Note that $r_0$ and $r_d$ are always set to $1$.
For simplicity, in this paper, we unify all other TT ranks (i.e., $r_1, \cdots, r_{d-1}$) to a single value denoted by $R$.
The representative optimization algorithm for  \ttd is TT-SVD~\cite{oseledets2011tensor}, which aims to obtain $\tensor{G}_1, \cdots, \tensor{G}_d$ satisfying $\fnorm{\tensor{X} - \tensor{\tilde{X}}_{TT}} \leq \epsilon \fnorm{\tensor{X}}$ for a prescribed accuracy $\epsilon$,
where $\tensor{\tilde{X}}_{TT}$ is the approximated tensor by TTD.
In TT-SVD, Truncated SVD is applied after reshaping a tensor to a matrix.
\ttd is naturally used as a lossy tensor compression algorithm, with the compressed results being the entries of the TT-core tensors. The number of these entries is $R^2\sum_{k=1}^d N_k= O(dNR^2)$, where $N$ represents the maximum mode length. 






\begin{table}[t]
\vspace{-5mm}
\caption{Symbol description}
\centering
     \begin{tabular}{c|l} 
         \toprule
         \textbf{Symbol} & \textbf{Description} \\ 
         \midrule
            $\mathbf{A} \in \mathbb{R}^{N_1 \times N_2}$ & $N_1$-by-$N_2$ matrix \\
            $\mathbf{A}(i,j)$ & $(i,j)$-th entry of $\mathbf{A}$ \\
            $\tensor{X} \in \mathbb{R}^{\shape}$ & tensor \\
            $d$ & order of $\tensor{X}$ \\
            $\tensor{X}(i_1, \cdots, i_d)$ & $(i_1, \cdots, i_d)$-th entry of $\tensor{X}$ \\        
            $N_{\text{max}}$ & maximum length of modes in $\tensor{X}$ \\
            $\slice{j}{i}$ & $i$-th slice of $\tensor{X}$ along the $j$-th mode\\
            $[n]$ & set of consecutive integers from $0$ to $n-1$ \\
            
            $\order = (\pi_{1}, \cdots, \pi_{d})$ & set of reordering functions for a tensor\\
            $\pi_{i}$ & reordering function for the $i$-th mode indices \\
            $\reorder{\tensor{X}}$ & reordered tensor of $\tensor{X}$ by $\order$\\
            $\fnorm{\tensor{X}}$ & Frobenius norm of $\tensor{X}$ \\
            $R$ & ranks of \ttc \\         
            $h$ & hidden dimension of LSTM \\
            $\model$ & set of the parameters of \nttd \\ 
            $\tensor{\Tilde{X}}$ & approximated tensor of $\tensor{X}$ by $\method$\\
            $\fold{\tensor{X}}$ & folded tensor of $\tensor{X}$ \\
            $d'$ & order of $\fold{\tensor{X}}$\\
            $\mathbf{T}_i$ & $i$-th TT core generated by $\model$ \\
         \bottomrule
     \end{tabular}
     \label{tab:symbol}
\end{table}

\subsection{Problem Definition}
We provide the formal definition of the lossy tensor compression problem addressed in this paper as follows:





\vspace{0.5mm}
\noindent\fbox{%
\parbox{0.97\columnwidth}{%
\vspace{-2mm}
\begin{problem}\label{prob:main}
\textsc{\normalfont{(Lossy Compression of a Tensor)}} 
\begin{itemize}[leftmargin=*]
    \item \textbf{Given:} a tensor $\mathcal{X} \in \mathbb{R}^{\shape}$,
    \item \textbf{Find:} the compressed data $D$ 
    \item \textbf{to Minimize:}
    (1) the size of $D$
    \\
    \indent \hspace{15mm} (2) the approximation error (e.g., $\fnorm{\tensor{X} - \tensor{Y}}^2$,\\
    \indent \hspace{17.5mm} where $\tensor{Y}$ is the tensor reconstructed from $D$)
\end{itemize}
\vspace{-2.5mm}
\end{problem}
}
}
\vspace{0.1mm}

In \method, the compressed data $D$ is composed of (a) the set $\model$ of parameters in Neural Tensor-Train Decomposition (\nttd), and (b) the set $\order$ of reordering functions.

\section{Proposed Method} \label{sec:method}
In this section, we propose \method, a precise and concise tensor-compression algorithm.

\subsection{Overview}
\label{sec:method:overview}

For our method, we focus on the following challenges in devising a compression algorithm based on \ttd:
\begin{enumerate}[leftmargin=6mm]
    \item [Q1.] \textbf{Expressiveness.}
    Since \ttd (see Section~\ref{sec:prelim:ttd}) relies only on linear operations of small-sized matrices, its expressiveness is limited.\footnote{For a more in-depth discussion of expressiveness, refer to \cite{oseledets2011tensor}.} How can we enhance its expressiveness to closely approximate a wider variety of tensors?
    
    \item [Q2.] \textbf{Conciseness.} 
    The number of parameters of \ttd increases with the order, mode length, and TT-rank.
    Can we further reduce the parameter number for additional compression?    
    \item [Q3.] \textbf{Data arrangement}.
    The approximation by \ttd depends on the arrangement of tensor entries since it shares \ttc based on mode indices. How can we re-arrange tensor entries to achieve a more accurate approximation?
\end{enumerate}

\method is based on the following ideas for addressing each of the aforementioned challenges:
\begin{enumerate}
    \item [A1.] \textbf{Neural Tensor-Train Decomposition (\nttd).}
    We incorporate an auto-regressive neural network into \ttd to improve its expressiveness while maintaining the number of parameters small (Section~\ref{sec:method:model}).
    \item [A2.] \textbf{Folding}. We fold the input tensor into a higher-order tensor to reduce the number of parameters and thus the size of the compressed output (Section~\ref{sec:method:folding}).
    \item [A3.] \textbf{Reordering}. We aid the model in better fitting the folded tensor by reordering the mode indices of the input tensor (Section~\ref{sec:method:reordering}).
\end{enumerate}

We summarize the overall process of \method in Algorithm~\ref{algo:overview}.
First, \method initializes the \nttd model $\model$ and the reordering functions $\order$.
Next, it reorders and folds $\tensor{X}$, creating a higher-order tensor $\fold{\reorder{\tensor{X}}}$.
Subsequently, the model parameters and reordering functions are updated to minimize the approximation error.
\method repeats this process until convergence is reached, meaning that the approximation error no longer exhibits significant changes.
The outputs of the compression process are optimized $\model$ and $\order$ based on which  each tensor entry is approximated in logarithmic time, as proven in Section~\ref{sec:method:theorem}.

To simplify the explanation, we assume that the input tensor ${\tensor{X}}$ is already properly ordered 
in Sections~\ref{sec:method:model} and \ref{sec:method:folding}. Afterward, we elaborate on how to initialize and update reordering functions $\order$ in Section~\ref{sec:method:reordering}.







\begin{figure}[t]    
    \centering
    \vspace{-4mm}
    \includegraphics[width=0.8\linewidth]{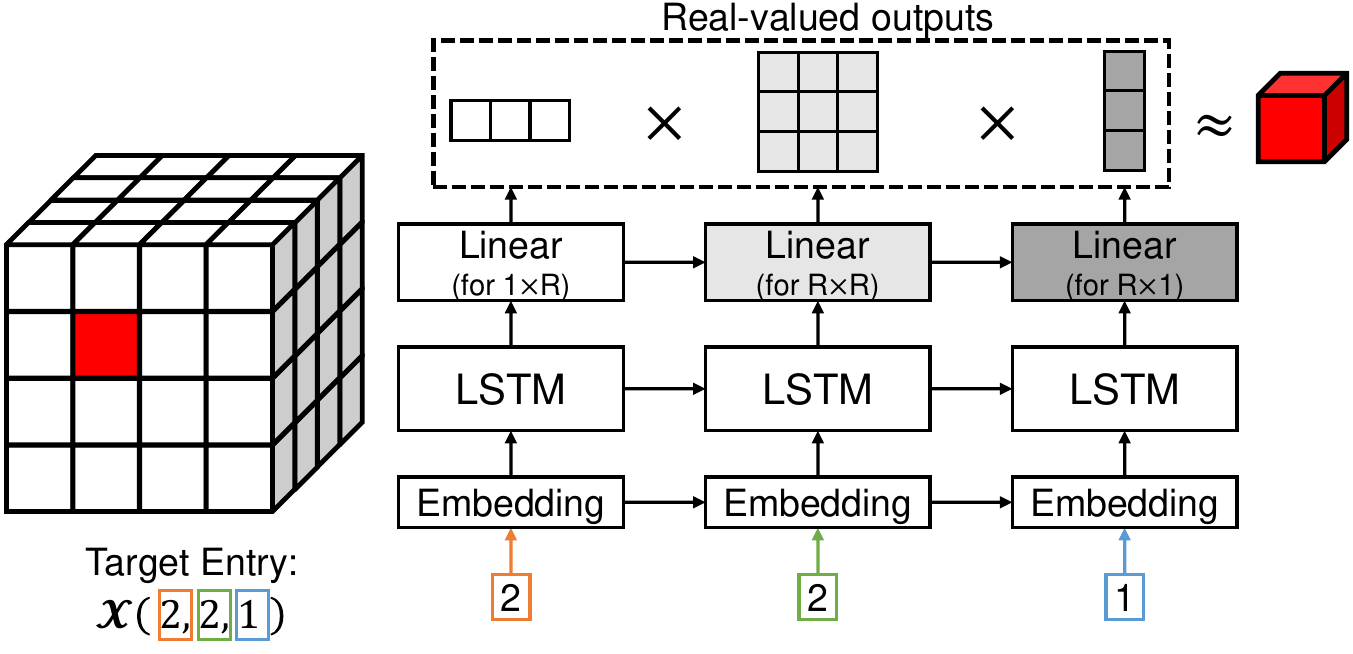}    
    \vspace{-2mm}
    \caption{The overall process of Neural Tensor-Train Decomposition (\nttd), which is the key component of \method.
    In order to generate \ttc for approximation, the mode indices of a target entry are encoded through an embedding layer and fed into an LSTM layer. 
    Each TT core is obtained through a linear layer from the LSTM output. An approximation is computed by the matrix product of the \ttc.
    Note that the parameters of \nttd are those of the neural networks, not \ttc themselves.}
    \label{fig:neutt_model}
\end{figure}

\begin{algorithm}[t]
\small
\caption{Overview of \method}\label{algo:overview}
\SetKwInput{KwInput}{Input}
\SetKwInput{KwOutput}{Output}
\KwInput{a tensor $\tensor{X} \in \mathbb{R}^{\shape}$}
\KwOutput{orders $\order$ and an NTTD model $\model$}
intialize $\model$ and $\order$ \label{algo:overview:init} \Comment*[f] 
{\small Section~\ref{sec:method:train:initorder} and Alg. 4 of \cite{appendix}} \\
 \While{${fitness}$ does not converge}{
     $\fold{\reorder{\tensor{X}}} \leftarrow $ reorder and fold $\tensor{X}$ \Comment*[f]{\small Section~\ref{sec:method:folding}} \\
    Update $\model$ \label{line:update:model} \Comment*[f]{\small Section~\ref{sec:method:model}} \\
    Update $\order$ \label{line:update:order} \Comment*[f]{\small Section~\ref{sec:method:train:order} and Alg.~\ref{algo:update:order}} \\
    $\tensor{\recon{X}} \leftarrow$ approximation using $\model$ and $\order$ \\
    ${fitness} \leftarrow 1 - \fnorm{\tensor{X} - \tensor{\recon{X}}} / \fnorm{\tensor{X}}$ \\
 }
 \Return $\model$ and $\order$
\end{algorithm}

\subsection{Neural TT Decomposition~(\nttd) for Fitting a Tensor}
\label{sec:method:model}
\method is a lossy compression algorithm, and its primary goal is to accurately approximate tensor entries with a small number of parameters. As a key component of \method, we propose Neural Tensor-Train Decomposition~(\nttd), a generalization of Tensor-Train Decomposition (TTD) that incorporates an auto-regressive neural network.

\begin{algorithm}[t] 
\small
\caption{Approximation of an entry by \nttd ($\theta$)} \label{algo:approx}
\SetKwInput{KwInput}{Input}
\SetKwInput{KwOutput}{Output}
\KwInput{(a) an index $(i_1, \cdots, i_{d})$ \\ 
(b) parameters of embedding layers ($ E_1, \cdots, E_{d}$) \\
(c) parameters of an \texttt{LSTM} layer\\
(d) parameters of linear layers ($\mat{W}_{\{1,d\}}$, $\mat{W}$, $\vect{b}_{\{1,d\}}$, $\vect{b}$)
}
\KwOutput{an approximated entry $\model(i_1, \cdots, i_{d})$}
\For{$k \leftarrow 1$ \normalfont{{to}} $d$\label{algo:approx:emb:for}}{
    $\vect{e}_k \leftarrow E_k(i_k)$ \label{algo:approx:emb}\\
}
$\vect{h}_1, \cdots, \vect{h}_{d} \leftarrow \texttt{LSTM}(\vect{e}_1, \cdots, \vect{e}_{d})$ \label{algo:approx:lstm}
\Comment*[f]{\small $\vect{h}_{k}\in \mathbb{R}^{h}$ for $1 \leq k \leq d$}
\\
$\mat{T}_1 \leftarrow \mat{W}_1\vect{h}_1 + \vect{b}_1$ \label{algo:approx:linear1}
\Comment*[f]{\small $\mat{W}_{1}\in \mathbb{R}^{1 \times R \times h}$ and $\vect{b}_{1}\in \mathbb{R}^{1 \times R}$}
\\
\For{$k \leftarrow 2$ \normalfont{{to}} $d-1$}{
    $\mat{T}_k \leftarrow \mat{W}\vect{h}_k + \vect{b}$ \label{algo:approx:linear2}
    \Comment*[f]{\small $\mat{W}\in \mathbb{R}^{R \times R \times h}$ and $\vect{b}\in \mathbb{R}^{R \times R}$}
    \\
}
$\mat{T}_{d} \leftarrow \mat{W}_{d}\vect{h}_{d} + \vect{b}_{d}$ 
\label{algo:approx:linear3}
\Comment*[f]{\small $\mat{W}_{d}\in \mathbb{R}^{R \times 1 \times h}$ and $\vect{b}_{d}\in \mathbb{R}^{R \times 1}$}
\\
 \Return $\mathbf{T}_1  \mathbf{T}_2  \cdots  \mathbf{T}_{d}$ \label{algo:approx:return}
\end{algorithm}

Instead of directly learning \ttc as free variables, for each $(i_1, \cdots, i_{d})$-th entry of the tensor, \ttc are obtained as the output of a neural network that takes the mode indices of the entry as an input.
The neural network, which we denote by $\model$, is trained to approximate the entry as follows:
\begin{equation} \label{eqn:approx}
    \tensor{X}(i_1, \cdots, i_{d}) \approx \model(i_1, \cdots, i_{d}) = \mat{T}_1 \mat{T}_2 \cdots \mat{T}_{d},
\end{equation}
where $\mat{T}_1 \in \mathbb{R}^{1 \times R}$, $\mat{T}_2 \in \mathbb{R}^{R \times R}$, $\cdots$, $\mat{T}_{d-1} \in \mathbb{R}^{R \times R}$, and $\mat{T}_{d} \in \mathbb{R}^{R \times 1}$ are the \ttc generated by $\model$.

Detailed procedures of \nttd are given in Algorithm~\ref{algo:approx} and illustrated in Figure~\ref{fig:neutt_model}. 
Note that \nttd consists of embedding, LSTM \cite{hochreiter1997long}, and linear layers.
To encode each mode index $i_k$, our model first looks up the embedding $\vect{e}_k$ from the embedding layer $E_k$ (lines~\ref{algo:approx:emb:for}~and~\ref{algo:approx:emb}).
Then, it feeds the embeddings $\vect{e}_{k}$ into the LSTM layer and obtains the hidden embeddings $\vect{h}_{k}$ for $1 \leq k \leq d$ (line~\ref{algo:approx:lstm}).
After generating TT cores $\mat{T}_{k}$ from $\vect{h}_{k}$ using the linear layers (lines~\ref{algo:approx:linear1}-\ref{algo:approx:linear3}), the product of the TT cores is returned as the approximated entry value (line~\ref{algo:approx:return})\footnote{To further reduce the model size, we share the same embedding parameters for modes of the same length (i.e., 
$E_k=E_j$ if $N_k=N_j$).
The effect of this approach is explored in Section \romanum{5} of \cite{appendix}.}.

Note that we employ an auto-regressive model to allow for the dependency between \ttc and mode indices. 
However, in \nttd, each $\mathbf{T}_k$ depends not only on the mode-$k$ index of the target entry (as in TTD) but also on its mode-$j$ indices for all $j\leq k$. 
While we employ LSTM \cite{hochreiter1997long}, any alternatives, such as GRU~\cite{cho2014learning} and Scaled Dot-product Attention~\cite{vaswani2017attention}, can be used instead.
In Section \Romannum{8} of \cite{appendix}, 
we examine the performance of \method equipped with alternatives.

\smallsection{Benefits over \ttd}
The \nttd model $\model$ has the following advantages over traditional \ttd: 
\begin{itemize}[leftmargin=*]
    \item {\textbf{Contextual:} In \nttd, as mentioned earlier, each TT core varies depending on not only the current mode index but also all preceding mode indices. 
    For example, consider approximating $\tensor{X}(2, 1, 2)$ and $\tensor{X}(1, 2, 2)$. 
    The TT cores used for the third mode in \nttd are different in these two cases. However, in TTD, the same TT core is used for both cases since the third mode indices are identical.
        By being contextual and non-linear (described below), \nttd is able to model tensors that cannot be easily approximated by TTD, even with more parameters, as shown experimentally in Section~\ref{sec:exp:power}. This improved expressiveness reduces the reliance on structural assumptions about input tensors.
    }
    \item {
        \textbf{Non-linear:}
        \nttd incorporates non-linear operations introduced by the LSTM layer, while TTD does not. This contributes to enhancing the expressiveness of \nttd, enabling \method to better approximate tensor entries. 
    }
    \item {
        \textbf{Concise:}
        \nttd shares parameters (specifically, $\mat{W}$ and $\mat{b}$ in line \ref{algo:approx:linear2} of Algorithm~\ref{algo:approx}) for different modes, which enables the model to be concise with fewer parameters than TTD. In contrast, TTD requires a unique TT core for each mode.
    }
\end{itemize}
We empirically examine the contribution of each component of \nttd to approximation accuracy in Section~\ref{sec:exp:ablation}.

\smallsection{Space complexity analysis}
We present the compressed output size of \nttd in Theorem~\ref{thm:space:nttd}. 
The hidden dimension of LSTM and the rank of TT cores are denoted by $h$ and $R$.
\begin{theorem}[Compressed Output Size of \nttD] \label{thm:space:nttd}
The size of the compressed output of \nttd (i.e., the number of its parameters) is $O(h(h + R^2+\sum_{i=1}^{d} N_i))$, which becomes $O(\sum_{i=1}^{d} N_i)$ if we treat $h$ and $R$ as constants.
\end{theorem}
\begin{proof}
The embedding layers have $O(\sum_{i=1}^d N_ih)$ parameters.
The LSTM and fully connected layers have $O(h^2 + hR^2)$ parameters.
Hence, the total size is $O(h(h + R^2+\sum_{i=1}^{d} N_i))$.
\end{proof}
\smallsection{Optimization method for $\model$}
We update the parameters $\model$ of \nttd using mini-batch gradient descent to minimize the loss function in Problem~\ref{prob:main}, i.e., $\fnorm{\tensor{X} - \tensor{Y}}^2$, where $\tensor{Y}$ is the tensor approximated by $\model$. 
As outlined in Section~\ref{sec:method:overview}, we alternately update $\model$ and the reordering functions $\pi$. After updating $\pi$, we reinitialize the optimizer (spec., Adam \cite{kingma2014adam}) since the loss surface changes after reordering, as detailed in Section~\ref{sec:method:train:order}.

\subsection{Folding Technique for Lightweight \nttd}
\label{sec:method:folding}
\begin{figure}[t]    
    \vspace{-3mm}
    \centering
    \includegraphics[width=\linewidth]{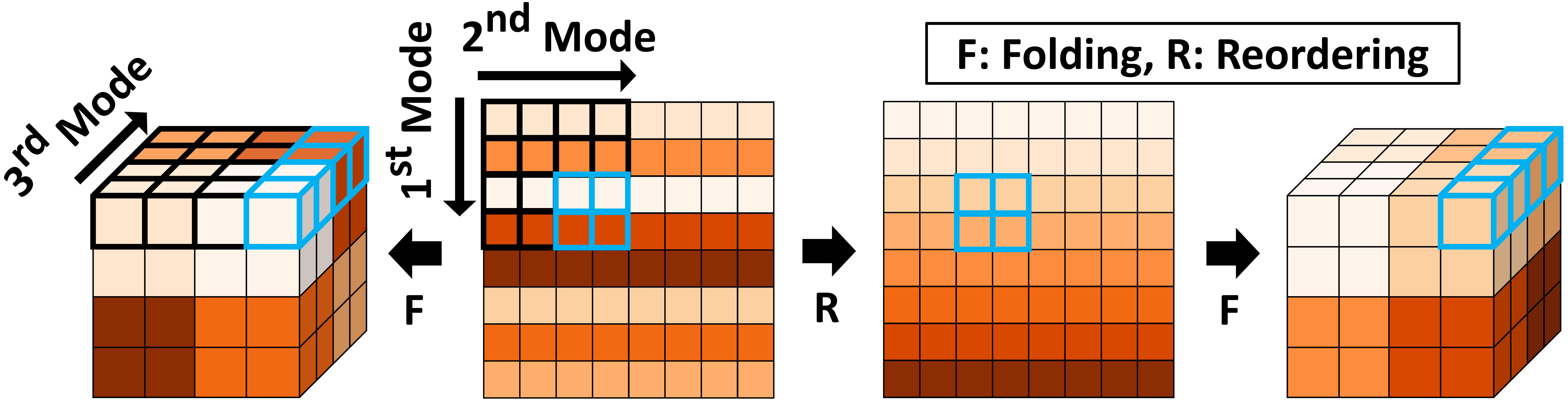}    
    \caption{An example of folding an 8$\times$8 matrix into a 4$\times$4$\times$4 tensor, where entries with similar values have similar colors. Note that when the rows of the matrix are reordered (R) such that adjacent rows are similar, the similar entries are located close to each other when the tensor is folded (F). In particular, similar entries share mode-$1$ and $2$ indices in the folded tensor.
}
    \label{fig:folding}
\end{figure}


%

We discuss the folding technique, which serves as the second component of \method. It aims to further reduce the size of the compressed output of \nttd, which is proportional to $\sum_{i=1}^d N_i$ according to Theorem~\ref{thm:space:nttd}. The main idea is to fold the input tensor into a higher-order tensor, maintaining the same number of entries but with smaller mode lengths.
In \method, the \nttd model $\model$ aims to fit the folded tensor rather than the input tensor after the arrangement process discussed in Section~\ref{sec:method:reordering}.



\smallsection{TT-matrix format}
Our folding technique is inspired by the TT-matrix format~\cite{oseledets2011tensor}, which aims to fold a matrix into a tensor for reducing the number of parameters in \ttd.
Given a matrix $\mathbf{A}$ of size $N \times M$ where $N=\prod_{k=1}^dn_k$ and $M=\prod_{k=1}^dm_k$, the format folds $\mathbf{A}$ into a $d$-order tensor $\tensor{A}$ of size $n_1m_1 \times \cdots \times n_dm_d$ (see the example with $d=3$ in Figure~\ref{fig:folding}) so that each entry of $\mathbf{A}$ is mapped to an entry of $\tensor{A}$ as follows:
\begin{equation*}
    \mathbf{A}(i, j) = \tensor{A}(i_1m_1 + j_1, \cdots, i_dm_d + j_d),
\end{equation*}
where $i_k \in [n_k] $ and $j_k \in [m_k]$ for each $1 \leq k \leq d$ are those satisfying $i=\sum_{k=1}^d i_k \prod_{l=k+1}^{d}n_l$ and $j=\sum_{k=1}^d j_k \prod_{l=k+1}^{d}m_l$.
The impact of this folding technique on \nttd is discussed below in the context of more general cases.
\smallsection{TT-tensor format}
While the TT-matrix format can be naturally extended to tensors, to the best of our knowledge, there has been no attempt to do so. In this work, we extend the TT-matrix format to tensors by folding the input tensor into a higher-order tensor with smaller mode lengths. We call this process the TT-tensor format.
Given a tensor $\tensor{X}$ of size $N_1 \times \cdots \times N_d$ where $N_k = \prod_{l=1}^{d'}n_{k,l}$, we fold $\mathcal{X}$ into a $d'$-order tensor $\fold{\tensor{X}}$ of size $\prod_{k=1}^d n_{k,1} \times \cdots \times \prod_{k=1}^d n_{k,d'}$. 
Then, the mapping between the entries of $\tensor{X}$ and $\fold{\tensor{X}}$ is:
\begin{align} 
    & \tensor{X}\Big( \sum_{k=1}^{d'} (i_{1,k} \prod_{l=k+1}^{d'}n_{1,l}), \cdots, \sum_{k=1}^{d'} (i_{d,k} \prod_{l=k+1}^{d'}n_{d,l}) \Big)  \label{eqn:idx:folded} \\
    &  \rightarrow \fold{\tensor{X}} \Big(\sum_{k=1}^d (i_{k,1} \prod_{l=k+1}^d n_{l,1}), \cdots, \sum_{k=1}^d (i_{k,d'} \prod_{l=k+1}^d n_{l,d'}) \Big), \nonumber
\end{align}
where $i_{k,l} \in [n_{k,l}]$ for all $k \in \{1, \cdots, d\}$ and $l \in \{1, \cdots, d'\}$.

In \method, we select the new order $d'$ such that the folded tensor has a higher order than the input tensor (i.e., $d' > d$), while maintaining $d' = O(\log N_{\text{max}})$, where $N_{\text{max}}$ represents the maximum mode length in the input tensor $\tensor{X}$.
In real-world tensors, this is usually feasible since their mode lengths are usually much larger than their orders. For instance, a 4-order tensor of size $256 \times 256 \times 256 \times 256$ can be folded into an 8-order tensor with each mode having a length of 16.
It may not always be feasible to construct a folded tensor that meets the above criteria while having the same number of entries as the input tensor. In such cases, the folded tensor may contain extra entries, the values of which are disregarded. For real-world tensors (see Section~\ref{sec:expr}), we initially assign $2$ to $n_{k,l}$ for all $k \in \{1, \cdots, d\}$ and $l \in \{1, \cdots, d'\}$ and modify some of them using integers at most $5$ to ensure that the input and folded tensors have similar numbers of entries.
For example, for the \texttt{PEMS-SF} dataset (a 3-order tensor of size $963\times 144 \times 440$), the assigned values in the form of a $d$ by $d'$ matrix are
\begin{equation*}
        \left[
            \begin{array}{*{20}c}
                2 & 2 & 2 & 2 & 2 & 2 & 2 & 2 & 2 & 2 \\
                2 & 2 & 2 & 2 & 2 & 5 & 1 & 1 & 1 & 1 \\
                2 & 2 & 2 & 2 & 2 & 2 & 2 & 2 & 2 & 1
            \end{array}
        \right],
\end{equation*} 
which result in a $10$-order tensor of size $8 \times 8 \times 8 \times 8 \times 8 \times 20 \times 4 \times 4 \times 4 \times 2$.
Note that $\prod_{l=1}^{d'}n_{1,l}=1024$, $\prod_{l=1}^{d'}n_{2,l}=160$, $\prod_{l=1}^{d'}n_{3,l}=512$ are close to $963$, $144$, and $440$, respectively.
See Section \romanum{4} of~\cite{appendix} for the values used for other datasets.


\smallsection{Space complexity analysis}
We analyze the effect of folding on the number of parameters (i.e., size of compressed output) in $\model$.
For simplicity, we assume $n_{k,l}=\sqrt[d']{N_k}$ for all  $k\in [d]$ and $l\in[d']$. 
According to Theorem~\ref{thm:space:nttd}, the number of parameters of \nttd of the original $\tensor{X}$ is
\begin{equation*}
O\Big(\sum_{k=1}^{d} \prod_{l=1}^{d'} n_{k,l} \Big) =O(N_1 + \cdots + N_d),
\end{equation*}
if we treat $h$ and $R$ as constants.

The number of parameters of \nttd of $\fold{\tensor{X}}$ is
\begin{equation*}
O\Big(\sum_{l=1}^{d'}\prod_{k=1}^d n_{k,l} \Big) = O\Big(d'\!\!\sqrt[d']{N_1 \cdots N_d} \Big),    
\end{equation*}
which is significantly  smaller than $O(N_1 + \cdots + N_d)$ in \nttd of the original tensor since
$O\left(d'\cdot\!\!\sqrt[d']{\prod_{k=1}^d{N_k}}\right) \in o(N_{\text{max}})$, where $N_{\text{max}}$ is the maximum mode length in $\tensor{X}$.
This is because $d'=O(\log N_{\text{max}})$, $d'>d$, and thus $O\left(d'\cdot\!\!\sqrt[d']{\prod_{k=1}^d{N_k}}\right) = O(N_{\text{max}}^e  \log N_{\text{max}})$ for some $e < 1$.
If we consider $R$ and $h$, according to Theorem~\ref{thm:space:nttd}, the space complexity becomes 
\begin{equation} \label{eq:space:folding}
O\Big(h\Big(h+R^2+d'\!\!\sqrt[d']{N_1 \cdots N_d} 
\Big)\Big).
\end{equation}

\subsection{Reordering Technique for Better Fitting a Folded Tensor}
\label{sec:method:reordering}

We present the reordering technique, the last component of \method.
Essentially, we reorder the mode indices in the input tensor before folding so that entries with similar values are placed close to each other by sharing mode indices in the folded tensor.
This arrangement improves the ability of our \nttd model $\model$, to fit the folded tensor more effectively, as $\model$ generates TT cores based on mode indices of target entries that serve as input for the model.


In the example in Figure~\ref{fig:folding}, the closer two entries are located in the original tensor, the more indices they tend to share in the folded tensor.
Specifically, the entries in the \textbf{\textcolor{black}{black}} region share only the first mode index in the folded tensor.
The adjacent entries in the \textbf{\textcolor{my_sky}{blue}} region share both the first and second mode indices  in the folded tensor.
It is important to note that in our model $\model$, the $k$-th \ttc $\mathbf{T}_k$ in Eq.~\eqref{eqn:approx} for approximating two entries are the same if their first $k$ indices are the same. Consequently, two TT cores are shared for the entries in the \textbf{\textcolor{my_sky}{blue}} region.
Therefore, the closer two entries are located in the original tensor, the more inputs and \ttc of $\model$ are likely to share for these entries. Due to this property, positioning similar entries close to each other helps $\model$ easily approximate entries more accurately.
We achieve such re-locations by reordering mode indices in the input tensor, as demonstrated in the example in Figure~\ref{fig:folding}, where the \textbf{\textcolor{my_sky}{blue}} region comprises more similar entries after reordering.

In \method, the mode-index reordering is accomplished by learning reordering functions $\order$. As outlined in Section~\ref{sec:method:overview}, we alternately update the model $\model$ and $\order$. In the following, we describe the initialization and update procedures for $\order$. It is worth noting that mode-index ordering by \method is associated with increasing smoothness, which is discussed in Section~\ref{sec:intro}, while many other compression methods presume that the input tensor is already highly smooth.




\smallsection{Initializing the orders} \label{sec:method:train:initorder}
We initialize the set $\order$ of reordering functions
using a surrogate loss function.
For all $k$, reordering the mode-$k$ indices (i.e., optimizing $\pi_k$) is formulated as:
\begin{equation} \label{eqn:sur_prob}
    \vspace{-2mm}
    \min_{\pi_k} \sum_{i=1}^{N_k-1} \left(\fnorm{\slice{k}{\pi_k(i)} - \slice{k}{\pi_k(i+1)}}\right), 
\end{equation}  
where $\slice{k}{i}$ is the $i$-th slice of $\tensor{X}$ along the $k$-th mode. 
Note that minimizing Eq.~\eqref{eqn:sur_prob} makes adjacent slices similar. 

The problem in Eq.~\eqref{eqn:sur_prob} can be reduced to Metric Travelling Salesman Problem~\cite{kao2008encyclopedia}.
Suppose each node represents a slice of the tensor, and each pair of nodes forms an edge with a weight equal to the Frobenius norm of the difference between their slices. 
Then, the optimal solution of the TSP problem on the resulting complete graph can be used to minimize Eq.~\eqref{eqn:sur_prob}.
However, since computing it is NP-hard, we instead obtain a 2-approximation solution based on the fact that the Frobenius norm satisfies the triangle inequality.  
Then, we delete the edge with the largest weight from the obtained solution, which is a TSP cycle, and set each $i$-th node in the resulting path to $\pi_k(i)$. 
Refer to Section \romanum{2} of \cite{appendix} for details on the TSP problem and the 2-approximation algorithm, including pseudocode.

%

\smallsection{Updating the orders based on $\model$ (Algorithm~\ref{algo:update:order})} \label{sec:method:train:order}
As outlined in Section~\ref{sec:method:overview},
after updating the \nttd model $\model$, we update the set $\order$ of reordering functions based on updated $\model$ and the loss function in Problem~\ref{prob:main}. This update step is described in Algorithm~\ref{algo:update:order}.
As defined in Section~\ref{sec:prelim:concepts}, we use $\reorder{\tensor{X}}$ to denote the tensor reordered from $\tensor{X}$ by $\order$. 
For each $k$-th mode, we consider $\floor{N_k/2}$ disjoint candidate pairs of mode-$k$ indices (lines 17-18). The process of obtaining candidate pairs is described in the following paragraph.
For each pair $(i, i')$ of mode indices, we consider the corresponding slices $\cslice{\reorder{\tensor{X}}}{k}{i}$ and $\cslice{\reorder{\tensor{X}}}{k}{i'}$. 
If swapping them reduces the loss function in Problem~\ref{prob:main}, we swap the values of $\pi_k(i)$ and $\pi_k(i')$ (lines~22-24).
Note that, since pairs are disjoint,
we can compute the changes in the loss and update $\pi_k$ in parallel using GPUs.

In the above process, each pair is composed so that swapping them tends to make similar slices located nearby in $\reorder{\tensor{X}}$.
We find such pairs using locality-sensitivity hashing (LSH) for Euclidean distance~\cite{leskovec2020mining}.
We sample half of the indices in each mode and vectorize the corresponding slices as points in a high-dimensional space.
We project the vectorized slices onto a random vector (lines~6-10) and evenly divide the projected points to create buckets.
We then repeatedly select two points in the same bucket.
Assuming that corresponding mode indices are $i_1$ and $i_2$, two pairs $(i_1, i_2 \oplus 1)$ and $(i_1 \oplus 1, i_2)$ are added as candidate pairs where $\oplus$ denotes the XOR operation (lines~17-18).
This approach aims to locate indices corresponding to similar slices nearby.
The remaining mode indices are paired randomly (lines~19-21).
The effectiveness of the reordering process is experimentally validated in Section~\ref{sec:exp:ablation}.

\begin{algorithm}[t]
\small
    \caption{Update of the reordering functions $\order$} \label{algo:update:order}
    \DontPrintSemicolon
    \SetKwInput{KwInput}{Input}
    \SetKwInput{KwOutput}{Output}
    \SetKwFunction{proc}{AddPairs}
    \SetKw{KwBy}{by}
    \KwInput{(a) a tensor $\tensor{X} \in \mathbb{R}^{\shape}$, (b) a \nttd model $\model$, \\ (c) orders $\pi_1, \cdots, \pi_d$}
    \KwOutput{updated orders $\pi_1, \cdots, \pi_d$}
    
    \For{$k$ $\leftarrow$ $1$ \KwTo $d$}{
        \tcp{\small Project each slice}
        $I \leftarrow \emptyset$ \\
        \For{$j \leftarrow 0$ \KwTo $N_k - 1$ \KwBy $2$}{
            $u \sim U(0, 1)$ \\
            \textbf{if} $u$ $<$ $1/2$ \textbf{then} $I$ $\leftarrow$ $I \cup \{j\}$ \textbf{else} $I$ $\leftarrow$ $I \cup \{j + 1\}$ \\
        }
        $s \leftarrow (\prod_{j=1}^d N_j)/N_k$; $P \leftarrow \emptyset$  \label{algo:update:order:project:start}\\
        $r \leftarrow$ a random point from $\mathbb{R}^{s}$ \\  
        \ForEach{$j \in I$}{
            $v \leftarrow$ vec$(\slice{k}{j})$ \\
            $P[j] \leftarrow r \cdot v / (\fnorm{r}\fnorm{v})$   \label{algo:update:order:project:end}
        }
        \tcp{\small Hash points to buckets}
        $num\_buckets \leftarrow \floor{N_k / 8}$ \\
        $bs \leftarrow \floor{(\max (P)- \min (P)) / num\_buckets}$; $B \leftarrow \emptyset$ \label{algo:update:order:bucket:start}\\              
        \ForEach{$j \in I$}{
            $bi \leftarrow \floor{(P[j]-\min(P))/bs}$ \\
            $B[bi] \leftarrow (B[bi] \cup \{j\}$) \label{algo:update:order:bucket:end}            
        }    
        \tcp{\small Build pairs (run in parallel)}
        $S \leftarrow \emptyset$; $S' \leftarrow \emptyset$ \\
        \ForEach{$b \in B$\label{algo:update:order:pair:start}} 
        {
            $\proc(b, S, True)$  \label{algo:update:order:pair:end}\Comment*[f]{\small Defined below}  \\
            $S' \leftarrow S' \cup b$  \label{algo:update:order:pair:random:start}
        }
        $S' \leftarrow S' \cup \{j \oplus 1 : j \in S'\}$ \\
        $\proc(S', S, False)$ \label{algo:update:order:pair:random:end} \\
        \tcp{\small Update orders (run in parallel)}
        \ForEach{$(i, i') \in S$\label{algo:update:order:start}}{
            $\Delta \leftarrow$ change in the loss when $\pi_k(i)$, $\pi_k(i')$ are exchanged   \\
            \lIf{$\Delta < 0$}{
                $\pi_k(i), \pi_k(i') \leftarrow \pi_k(i'), \pi_k(i)$\label{algo:update:order:end}
            }
        }

    }
    \Return $\pi_1, \cdots, \pi_d$

    \SetKwProg{myproc}{Function}{}{}
    \myproc{$\proc{C, S, xor}$}{
      \While{$|C| > 1$}{
        Randomly sample $(i_1, i_2)$ from $C$ \\
        \lIf{$xor$}{$S \leftarrow S \cup \{(i_1, i_2 \oplus 1)\} \cup \{(i_1 \oplus 1, i_2)\}$}
        \lElse{$S \leftarrow S \cup \{(i_1, i_2)\}$}
        $C \leftarrow C \setminus \{i_1, i_2\}$
      }
    }
\end{algorithm}
\vspace{-1mm}




\subsection{Theoretical Analysis} \label{sec:method:theorem}
We theoretically analyze \method's compressed-output size, entry-reconstruction speed, and compression speed.
For simplicity,
we assume that all mode sizes of the input tensor $\tensor{X} \in \mathbb{R}^{N_1 \times \cdots \times N_d}$ are powers of 2 (i.e., $n_{l,k} \in \{1, 2\}$ for all $l \in \{1, \cdots, d\}$ and $k \in \{1, \cdots, d'\}$ in Section~\ref{sec:method:folding}).
We use $N_{\text{max}}$ to denote the maximum size of modes in $\tensor{X}$, and use $h$ and $R$ to denote the hidden dimension and TT rank of our model.

\smallsection{Size of compressed outputs}
We present the space complexity of the  outputs produced by \method in Theorem~\ref{thm:space:method}.
It is important to note that the complexity is much lower than $O(\prod_{i=1}^{d} N_i)$ of the original tensor and can also be lower than $O(R^2\sum_{i=1}^d N_i)$ of \ttd and $O(R\sum_{i=1}^d N_i)$ of CP Decomposition (CPD),
especially for large values of $R$.

\begin{theorem}[Size of Compressed Outputs] \label{thm:space:method}
The size of the compressed output $D = (\model, \order)$ produced by \method (i.e., Algorithm~\ref{algo:overview}) is $O(h(2^d + h + R^2) + \sum_{i=1}^{d} N_i \log N_i)$. 
\end{theorem}
\begin{proof}
The embedding layer of $\model$ has $O(h2^d)$ parameters because our model shares the embedding layers across different modes in $\fold{\reorder{X}}$ and the largest mode size of $\fold{\reorder{X}}$ is $2^d$, as detailed in Section~\ref{sec:method:folding}.
The numbers of parameters of LSTM and fully connected layers are $O(h^2 + hR^2)$ since the number of parameters of each linear layer is proportional to the product of the input dimension and the output dimension.
For each mode $i$, the number of all possible orderings of $\pi_i$ is $N_i!$.
Thus, we need $O(\log N_i!) \in O(N_i \log N_i)$ bits to save one among them. 
Hence, the total size of the compressed output is $O(h(2^d + h + R^2) + \sum_{i=1}^{d} N_i \log N_i)$. 
\end{proof}

\smallsection{Speed of reconstruction}
Another important aspect of a compression algorithm is the speed of reconstruction.
Theorem~\ref{thm:time:entry} formalizes the reconstruction speed for the output of \method. While the complexity is higher than $O(dR^2)$ of \ttd or $O(dR)$ of CPD, it is only logarithmic in mode lengths.

\begin{theorem}[Reconstruction Speed] \label{thm:time:entry}
Given the output $D=(\model, \order)$ of \method (i.e., Algorithm~\ref{algo:overview}),
approximating the value of an input tensor entry (i.e., Algorithm~\ref{algo:approx} on $\reorder{\tensor{X}}$)  takes $O((d + h^2 + hR^2)\log N_{\textnormal{max}})$ time.
\end{theorem}
\begin{proof}
For each entry of $\tensor{X}$, earning the mode indices in $\reorder{\tensor{X}}$ requires $O(d)$ time. 
Computing the mode indices in $\fold{\reorder{\tensor{X}}}$ requires $O(d\log N_{\text{max}})$ time because all $i_{kl}$ in  Eq.~\eqref{eqn:idx:folded} must be computed where $k$ ranges from 1 to $d$ and $l$ ranges from 1 to $d'=O(\log N_{\text{max}})$.
Processing the inputs through the embedding layer and LSTM in $\model$ requires $O(h^2 \log N_{\text{max}})$ time.
The time complexity of computing TT cores with fully connected layers is $O(hR^2\log N_{\text{max}})$, and that of computing products of TT cores is $O(R^2 \log N_{\text{max}})$, provided that the order of computations is optimized.
Therefore, the total time complexity of approximating each entry is $O((d + h^2 + hR^2)\log N_{\text{max}})$.
\end{proof}

\smallsection{Speed of compression}
We analyze the speed of the compression process of \method.

\begin{theorem}[Compression Speed] \label{thm:time:total}
The time complexity of \method (i.e., Algorithm~\ref{algo:overview}) with $T$ update steps is $O((Td(d + h^2 + hR^2)\log N_{\textnormal{max}} + \sum_{i=1}^{d}N_i)\prod_{i=1}^{d} N_i)$, where $\prod_{i=1}^{d} N_i$ is the number of tensor entries.
\end{theorem}
\noindent \textit{Proof Sketch.} Initializing all reordering functions $\order$ takes $O((\sum_{i=1}^{d}N_i)\cdot(\prod_{i=1}^{d}N_i))$ time.
Updating $\model$ and $\pi$ once (i.e., lines~\ref{line:update:model}-\ref{line:update:order}) takes
$O(d(d + h^2 + hR^2)\prod_{i=1}^{d} N_i \log N_{\textnormal{max}})$ time.
For a full proof, see Section \romanum{1} of \cite{appendix}. \qed
\vspace{0.5mm}


\smallsection{Connection to actual running time}
In Sections~\ref{sec:exp:scalability}, we measure the actual running time for compression and reconstruction to confirm the time complexity.
In practice, the term $\prod_{i=1}^{d} N_i$, which corresponds to the number of entries, is much larger and also increases much faster than all other terms, and as thus
the compression time increases near linearly with it.


\smallsection{Memory requirements}
The complexity of memory space required for compression by \method does not exceed the combined memory requirements for a mini-batch, the compressed output, and the reordering functions.
\begin{theorem}[Memory Requirements for Compression] \label{thm:space:train}
\method (Algorithm 1) requires $O(Bd + h(2^d + B(h + R^2) \log N_{\textnormal{max}}) + \sum_{i=1}^d N_i)$ memory space, where $B$ is the number of tensor entries in a mini-batch.
\end{theorem}
\begin{proof}
    Refer to Section \romanum{1} of \cite{appendix}.
    \vspace{-1mm}
\end{proof}

\section{Experiments} \label{sec:expr}

We review our experiments for the following questions:
\begin{enumerate}[leftmargin=6mm]
    \item [Q1.] \textbf{Compression Performance:} 
    How accurately and compactly does \method compress real-world tensors?
    \item [Q2.] \textbf{Ablation Study:} 
    How effective is each component of \method for reconstruction accuracy?     
    \item [Q3.] \textbf{Scalability:}
    How does \method's compression and reconstruction time scale with input tensor size?
    \item [Q4.] 
    \textbf{Further Inspection:} Can our method accurately compress high-rank tensors? How does it reorder mode indices?  
   \item [Q5.] \textbf{Compression Time:}
    How does the speed of compression by \method compare to that of its competitors?
    \item [Q6.] \textbf{Effect of Hyperparameters (Section~\romanum{7} of \cite{appendix}):} How sensitive is \method to hyperparameter settings?
\end{enumerate}

\setlength{\tabcolsep}{2.5pt}
\begin{table}[t]
    \vspace{-1mm}
    \centering
    \caption{
    Statistics of public real-world datasets used in the paper.} 
    \label{tab:datasets}
    \scalebox{0.85}{
    \begin{tabular}{c|c|c|c|c|c}
        \toprule
        \multirow{2}{*}{\textbf{Name}} & \multirow{2}{*}{\textbf{Size}} & \textbf{Order} & \textbf{Order} & \multirow{2}{*}{\textbf{Density}} & \multirow{2}{*}{\textbf{Smoothness}} \\
        & & (Original) & (Folded) & \\
        \midrule
        \texttt{Uber} & $183 \times 24 \times 1140$ & \multirow{6}{*}{3}  & 9 & 0.138 & 0.861\\
        \texttt{Air Quality} & $5600 \times 362 \times 6$ &  & 12 & 0.917 & 0.513 \\
        \texttt{Action} & $100 \times 570 \times 567$ &  & 8 & 0.393 & 0.484 \\
        \texttt{PEMS-SF} & $963 \times 144 \times 440$ &  & 10 & 0.999 & 0.461 \\
        \texttt{Activity} & $337 \times 570 \times 320$ & & 8 & 0.569 & 0.553 \\
        \texttt{Stock} & $1,317 \times 88 \times 916$ & & 10 & 0.816 & 0.976 \\        
        \midrule
        
        \texttt{NYC} & \makecell{$265 \times 265 \times 28  \times 35$} & \multirow{2}{*}{4} & 7 & 0.118 & 0.788 \\ 
        \texttt{Absorb} & \makecell{$192 \times 288 \times 30 \times 120$} &  & 7 & 1.000 & 0.935 \\
        \bottomrule
    \end{tabular}
    }
\end{table}

\begin{figure*}[t]
    \vspace{-4mm}
    \begin{minipage}[c][7cm]{0.1\linewidth}
        \centering
        \includegraphics[width=\linewidth]{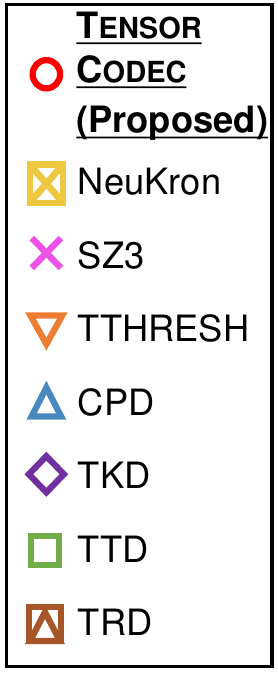} 
    \end{minipage}
    \begin{minipage}[c][7cm][b]{0.9\linewidth} 
        \centering
        \subfigure[\texttt{Stock}]{
            \includegraphics[width=0.24\linewidth]{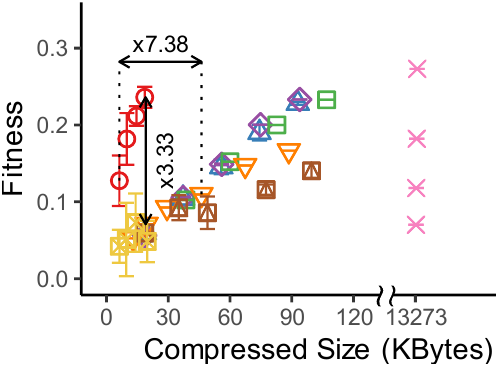}
            \label{fig:trade_off:kstock}
        }
        \hspace{-4mm}
        \subfigure[\texttt{Activity}]{
            \includegraphics[width=0.24\linewidth]{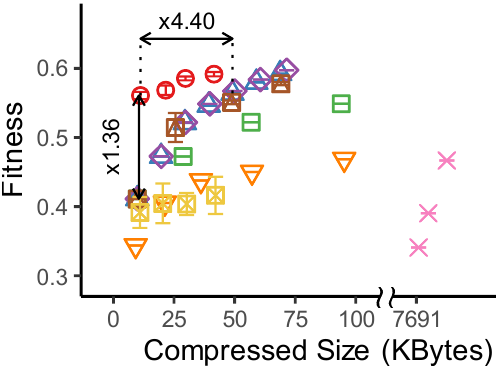}    
        }
        \hspace{-4mm}
        \subfigure[\texttt{Action}]{
            \includegraphics[width=0.24\linewidth]{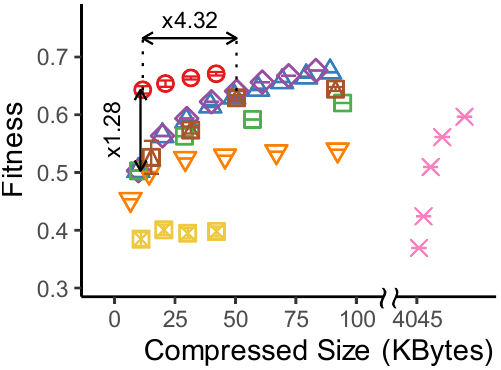}    
        }
        \hspace{-4mm}
        \subfigure[\texttt{Air Quality}]{
            \includegraphics[width=0.24\linewidth]{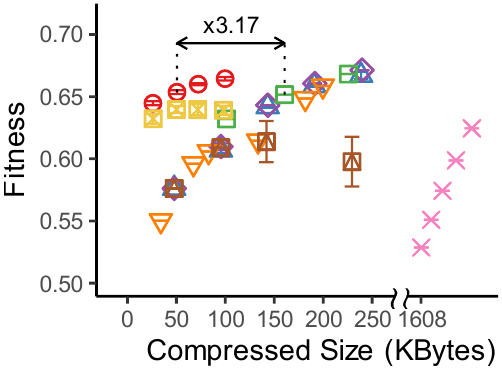}  
        }       
        \subfigure[\texttt{PEMS-SF}]{
            \includegraphics[width=0.24\linewidth]{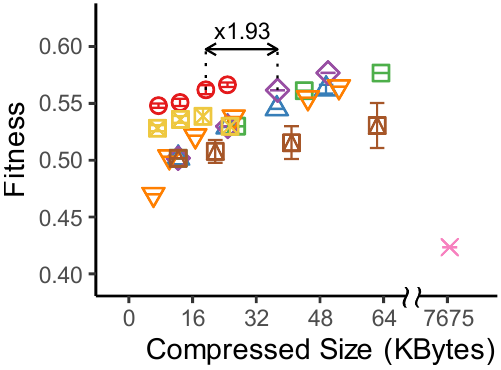}  
        }
        \hspace{-4mm}
        \subfigure[\texttt{Uber}]{
            \includegraphics[width=0.24\linewidth]{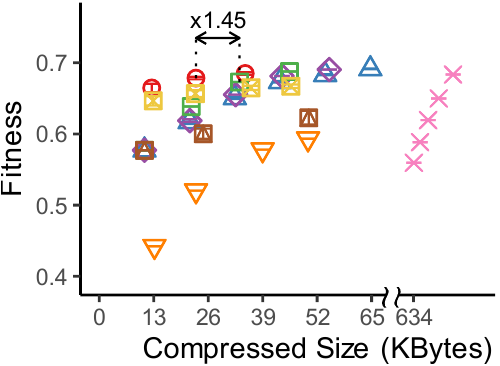}  
        }
        \hspace{-4mm}
        \subfigure[\texttt{Absorb}]{
            \includegraphics[width=0.24\linewidth]{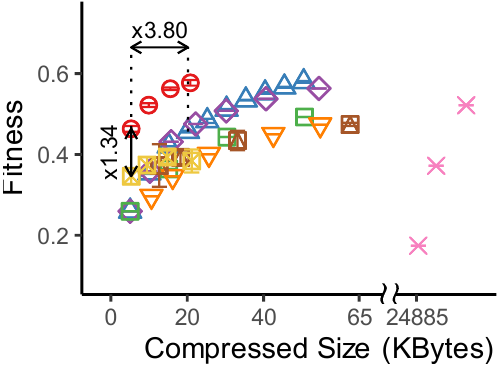}    
        }
        \hspace{-4mm}
        \subfigure[\texttt{NYC}]{
            \includegraphics[width=0.24\linewidth]{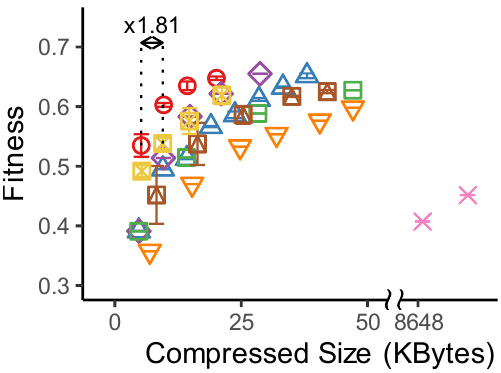}  
        } \\
    \vspace{-2mm}
    \end{minipage} 
    
    \caption{\label{fig:trade_off} \underline{\smash{\method provides compact and accurate tensor compression.}}
    Its compressed size is up to 7.38$\times$ smaller than that of the best  competitor showing similar fitness.
    Given a budget for compressed size, \method shows up to 3.33$\times$ higher fitness than the best baseline.
    }
\end{figure*}

\subsection{Experimental Settings}
\label{sec:exp:settings}

\smallsection{Machine} 
For \method and NeuKron, which require GPUs, we used a machine with 4 RTX 2080Ti GPUs and 128GB RAM.
For all other algorithms, we used a machine with an i5-9600K (6 cores) and 64GB RAM. 
Note that the outputs are independent of machine specifications.

\smallsection{Datasets} 
We used 8 real-valued tensors from real-world datasets to verify the effectiveness of our method. 
We provide basic statistics of the tensors in Table~\ref{tab:datasets}. 
Smoothness is defined as $1-\frac{\mathbb{E}_i[\sigma_3(i)]}{\sigma}$, where $\sigma_3(i)$ denotes the standard deviation of values in the sub-tensor of size $3^d$ centered at the position $i$, and $\sigma$ denotes the standard deviation of the values of the entire tensor.
\textbf{Note that we used tensors with diverse properties, including varying density and smoothness, from different domains, without any assumption on data property.}
See Section~\romanum{3} of \cite{appendix} for data semantics and pre-processing steps. 

\smallsection{Competitors}
We compared \method with 7 state-of-the-art methods (refer to Section~\ref{sec:rel_works}) to evaluate their performance in terms of tensor compression. 
We employed four well-known tensor decomposition methods: \textbf{CPD}~\cite{carroll1970analysis}, \textbf{TKD}~\cite{tucker1966some}, \textbf{\ttd}~\cite{oseledets2011tensor}, and \textbf{TRD}~\cite{zhao2019learning} (Section~\ref{sec:rel_works}). 
These methods are implemented in MATLAB R2020a, which offers a highly efficient linear algebra library.\footnote{Tensor Toolbox. https://tensortoolbox.org/.} \footnote{TT-Toolbox. https://github.com/oseledets/TT-Toolbox.} \footnote{https://github.com/oscarmickelin/tensor-ring-decomposition}
We used the official C++ implementation of \textbf{TTHRESH}~\cite{ballester2019tthresh} and \textbf{SZ3}~\cite{zhao2021optimizing}, which are for compressing tensors of visual data.
Since the compression results of TTHRESH and SZ3 depend on mode-index orders, we applied our order-initialization technique to these methods only when it was helpful.
We also used the official PyTorch implementation of \textbf{NeuKron}~\cite{kwon2023neukron}, the state-of-the-art deep-learning-based method for compressing sparse tensors.
We implemented \method using PyTorch.

\smallsection{Experimental Setup} 
To measure the accuracy of the compression methods, we employed fitness, which is defined as:
\begin{equation*}
    \vspace{-1mm}
    {fitness} = 1-({\fnorm{\tensor{X} - \hat{\tensor{X}}}}/{\fnorm{\tensor{X}}})
\end{equation*}
where $\hat{\tensor{X}}$ is an approximation of $\tensor{X}$ obtained by the evaluated method. 
This metric has been widely used to assess the approximation accuracy of tensor decomposition methods~\cite{battaglino2018practical, perros2018sustain}. 
The fitness value is smaller than 1, with higher fitness indicating more accurate approximation.
For comparing compressed sizes,
we used the double-precision floating-point format for all methods.
We encoded the order of $N_k$ indices in each $k$-th mode using $N_k \log_2N_k$ bits, by representing each integer from 0 to $N_k-1$ in $\log_2 N_k$ bits.
Orders of all mode indices were stored separately in \method and NeuKron (and also in SZ3 and TTHRESH when our reordering scheme was applied to them).
We conducted each experiment 5 times to obtain the average and standard deviation of running time and fitness.
Refer to Section \romanum{4} of \cite{appendix} for more details.

\subsection{Q1. Compression Performance} \label{sec:exp:tradeoff}
We evaluated the compression performance of \method and its competitors, focusing on the trade-off between compressed size and fitness.
The hyperparameters of the compared methods were configured to yield similar compressed sizes, as detailed in Section~\romanum{4} of \cite{appendix}.
A time limit of 24 hours was set for running each method.

\textbf{\method outperforms all competitors by offering a superior trade-off between compressed size and fitness, as shown in Figure~\ref{fig:trade_off}}.
Particularly, on the \texttt{Stock} dataset, \method achieves a compressed size 7.38$\times$ smaller than that of the second-best method, while offering similar fitness. Moreover,
when the compressed size is almost the same, the fitness of \method is 3.33$\times$ greater than that of the competitor with the best fitness.
While NeuKron is specifically designed for sparse tensors and performs relatively better in such datasets, including the \texttt{Uber} dataset, \method still outperforms it on all datasets.


\begin{figure}[t]
    \centering
    \vspace{-1mm}
    \includegraphics[width=0.75\linewidth]{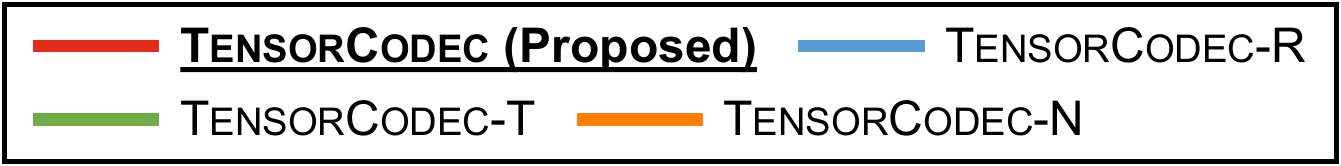}
    
    \subfigure[\texttt{Action}]{
        \includegraphics[width=0.24\linewidth]{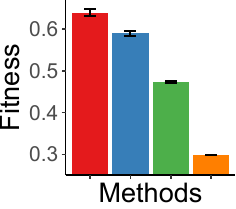} 
    }
    \hspace{-5mm}
    \subfigure[\texttt{Air Quality}]{
        \includegraphics[width=0.24\linewidth]{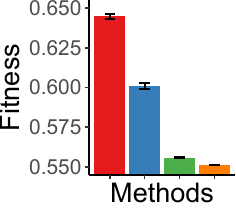} 
    }
    \hspace{-5mm}
    \subfigure[\texttt{PEMS-SF}]{
        \includegraphics[width=0.24\linewidth]{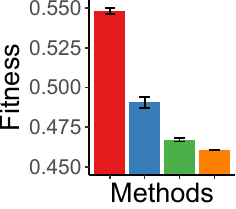} 
    }
    \hspace{-5mm}
    \subfigure[\texttt{Uber}]{
        \includegraphics[width=0.24\linewidth]{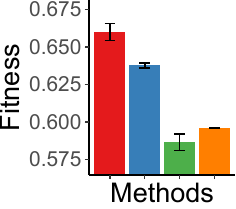} 
    } \\
    \vspace{-2mm}
    \caption{\underline{\smash{Every component of \method is effective for improved compression.}} Fitness significantly increases as more components are added. \label{fig:ablation}
    }
\end{figure}

\subsection{Q2. Ablation Study}
\label{sec:exp:ablation}
To demonstrate the effectiveness of each component of \method, we considered the following variants of it:
\begin{enumerate}[leftmargin=*]
\item \method-R: a variant of \method that does not include the repeated {\bf r}eordering of mode indices.
\item \method-T: a variant of \method-R that does not initialize the order of mode indices based on the $2$-approximate solution of Metric {\bf T}SP.
\item \method-N: a variant of \method-T without an auto-regressive {\bf n}eural network for generating TT cores. Instead, it simply applies \ttd to the folded tensor.
\end{enumerate}


We compared the reconstruction accuracy in terms of fitness for the four small datasets, 
to evaluate the performance of the \method variants.
\method-N is optimized by TT-SVD, while the other variants are optimized by gradient descents.
We set the TT ranks of \method-N so that its number of parameters is closest to that of the other variants, which have exactly the same number of parameters.

\textbf{Each component of \method contributes to improved compression.}
As shown in Figure~\ref{fig:ablation}, the fitness of \method increases as more components are added.
An exception occurs in the \texttt{Uber} dataset, where the fitness shows a slight increase from \method-T to \method-N. This is likely to be influenced by the difference in the number of parameters (i.e., \method-N uses 10,776 bytes while \method-T uses 10,104 bytes for parameters).
\vspace{-1mm}

\begin{figure}[t]
    \centering
    \vspace{-3mm}
    \includegraphics[width=0.65\linewidth]{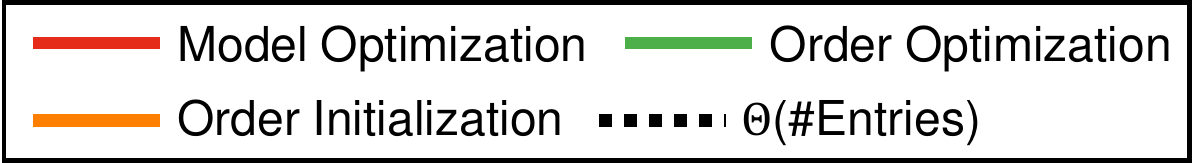}
    \subfigure{    
        \includegraphics[width=0.3\linewidth]{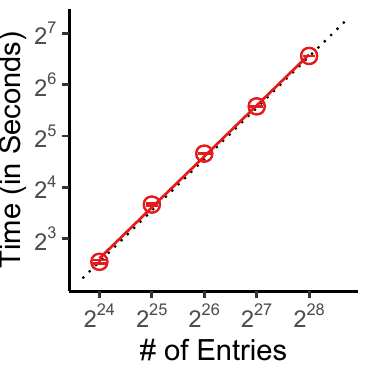}
        \hspace{-2mm}
        \includegraphics[width=0.3\linewidth]{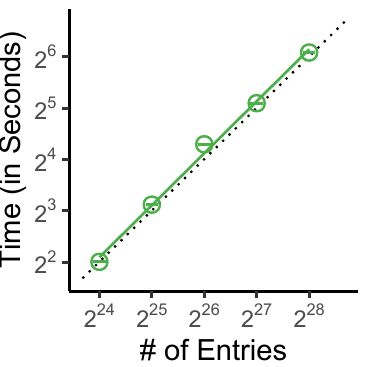}
        \hspace{-2mm}
        \includegraphics[width=0.3\linewidth]{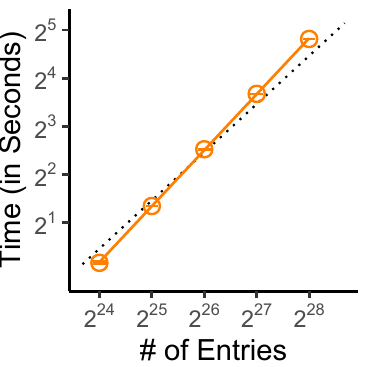}
    } \\
    \vspace{-2mm}
    \caption{
    \underline{\smash{The compression time of \method scales near linearly with}}
    \underline{\smash{the number of entries in the input $4$-order tensor.}}
    See Section \romanum{6} of \cite{appendix} for the near-linear scalability on $3$-order tensors.}
    \label{fig:scalable:train}
\end{figure}

\subsection{Q3. Scalability} \label{sec:exp:scalability}

We investigate the scalability of \method with respect to the tensor size while fixing $R$ and $h$ values to 8.

\smallsection{Compression speed}
We examined the increase in compression time for \method as the number of entries in the input tensors grows. 
Specifically, we measured the time taken for order initialization and a single iteration of the model and order optimization on five synthetic full tensors with varying sizes. 
Their sizes can be found in Table~\romanum{5} of \cite{appendix}. 
Each entry of them is sampled uniformly between 0 and 1.

\textbf{The compression time of \method, specifically, all its three steps. increases near linearly with the number of entries in the input tensor}, as illustrated in Figure~\ref{fig:scalable:train}. 
This is because the number of tensor entries is much larger and also increases much faster than all other terms in the time complexity in Theorem~\ref{thm:time:total} in Section~\ref{sec:method:theorem}.

\smallsection{Reconstruction speed}
We also examined the scalability of reconstruction from the compressed output of \method. We generated synthetic tensors of orders 3 and 4, with mode sizes increasing from $2^6$ to $2^{18}$ by a factor of 2. Then, we measured the total elapsed time for reconstructing $2^{18}$ entries sampled uniformly at random from each tensor.

\textbf{The reconstruction time of \method is sublinear} as shown in Figure~\ref{fig:scalable:infer}.
Specifically, the time increases logarithmically with the size of the largest mode in the input tensor. This result is in line with Theorem~\ref{thm:time:entry} presented in Section~\ref{sec:method:theorem}.

\begin{figure}[t]
    \vspace{-5mm}
    \centering
    \subfigure[3-order tensor]{
\includegraphics[width=0.35\linewidth]{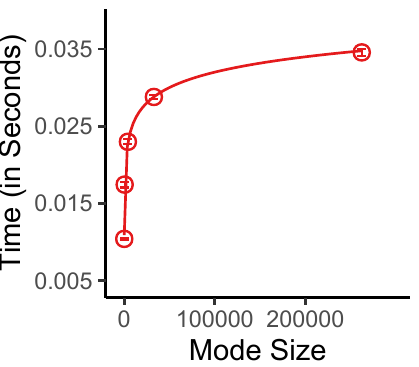}
    }
    \subfigure[4-order tensor]{
\includegraphics[width=0.35\linewidth]{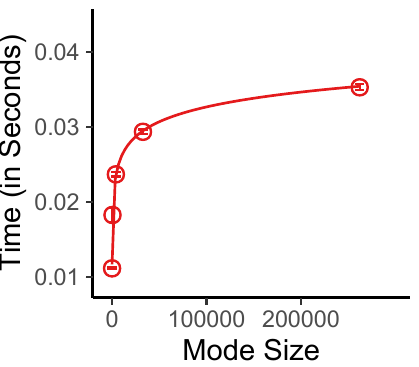}
    } \\
        \vspace{-2mm}
    \caption{ \label{fig:scalable:infer} 
         \underline{\smash{The reconstruction time of \method is sub-linear.}} 
         When reconstructing the input tensor entries from the outputs of \method, the required time is logarithmic with respect to the largest mode size.
    }
\vspace{-1mm}    
\end{figure}

\subsection{Q4. Further Investigation}
\label{sec:exp:power}

\smallsection{Order of mode indices}
We compare the mode index orders obtained by \method and NeuKron in the \texttt{NYC} dataset.\footnote{We set $R$ to 10 and $h$ to 9 for \method. For NeuKron, we set $h$ to $14$ so that it uses slightly more space than \method. However, NeuKron's compression did not terminate within 24 hours, so we used the result obtained at the time limit.}
The indices of the first two modes of the dataset indicate regions of New York City (see Section~\romanum{3} of \cite{appendix}).

As visualized in Figure~\ref{fig:case_study}.
\textbf{the reordering process of \method assigns nearby locations to similar mode indices.} This is particularly evident in Manhattan and Staten Island, where the colors of nearby locations are similar to each other. However, the mode indices obtained by NeuKron do not exhibit any clear patterns.
This result demonstrates the effectiveness of the reordering technique used in \method.

\begin{figure}[t]
    \vspace{-5mm}
   \centering
    \subfigure[\label{fig:case_study:tencodec} Reordering by \method]{
        \includegraphics[width=0.47\linewidth]{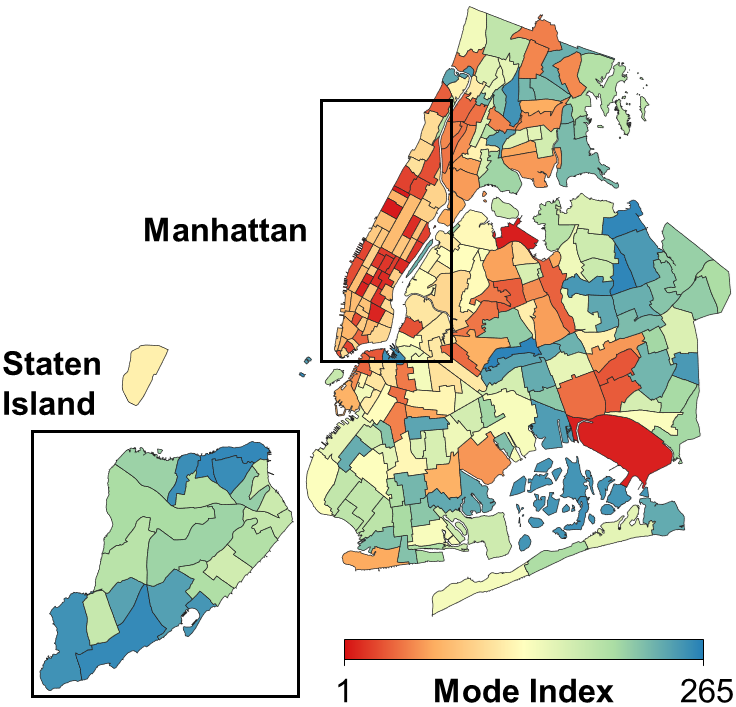}
    }    
    \hspace{-2mm}
    \subfigure[\label{fig:case_study:neukron} Reordering by \textsc{NeuKron}]{
        \includegraphics[width=0.47\linewidth]{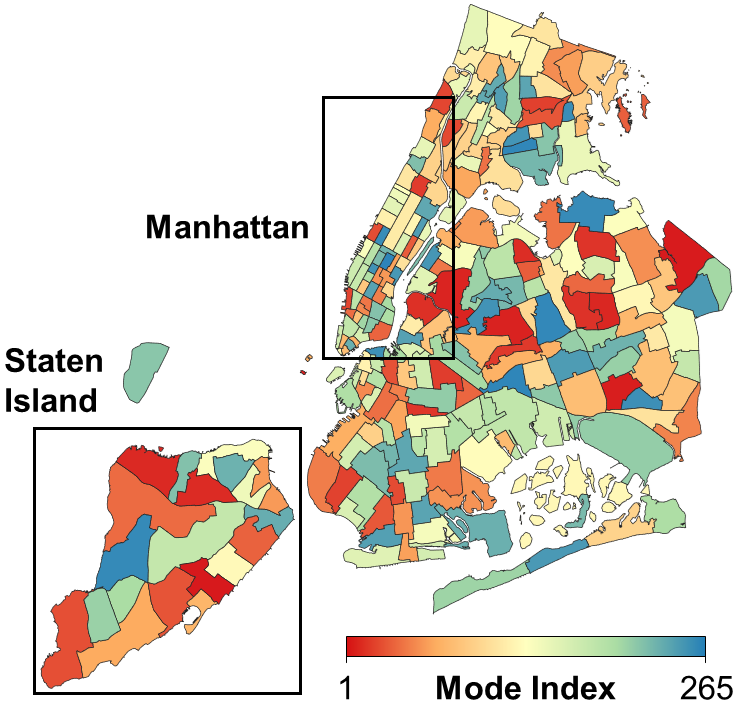}
    } \\
    \vspace{-2mm}
    \caption{
    In the \texttt{NYC} dataset, where mode indices correspond to locations within New York City, 
    \method  assigns nearby locations to similar mode indices, as shown in (a), resulting in comparable colors for nearby areas, particularly in Manhattan and Staten Island. However, the order obtained by NeuKron in (b) does not exhibit any clear patterns.
    } 
    \label{fig:case_study}
\end{figure}

\smallsection{Expressiveness}
We evaluated the expressiveness of \method by analyzing the number of parameters required for typical tensor decomposition methods to accurately approximate tensors generated by \method. To conduct the experiment, we created two tensors, one with size $256\times256\times256$ and another with size $128\times128\times128\times128$, by unfolding higher-order tensors generated by \method.\footnote{The TT rank $R$ and hidden dimension $h$ of \nttd were both set to $5$, and the parameters of \nttd were randomly initialized.}

The results, depicted in Figure~\ref{fig:high_rank}, demonstrate that \textbf{\method is capable of expressing high-rank tensors that require a large number of parameters for typical tensor decomposition methods to approximate accurately.}
\vspace{-2mm}
\begin{figure}[t]
    \vspace{-6mm}
    \hspace{3mm}
    \begin{minipage}[c][3cm]{0.18\linewidth}
        \includegraphics[width=\linewidth]{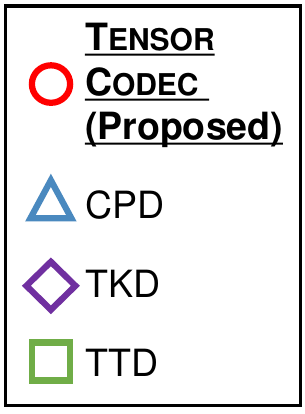}
    \end{minipage}
    \begin{minipage}[c][3cm][b]{0.8\linewidth} 
        \subfigure[3-order tensor]{
            \includegraphics[width=0.4\linewidth]{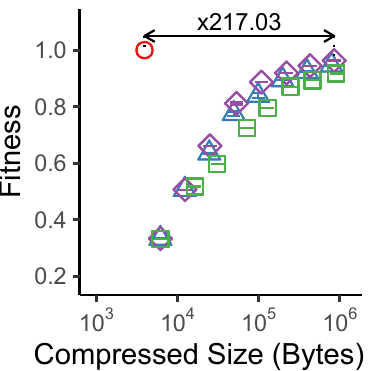}
        }
        \subfigure[4-order tensor]{
            \includegraphics[width=0.4\linewidth]{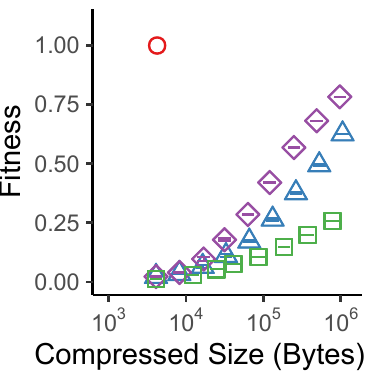}
        }
    \end{minipage}  
    \\
    \vspace{-3mm}
    \caption{\label{fig:high_rank}
    Expressiveness of \method. It is capable of generating high-rank tensors, which require a large number of parameters to be precisely approximated by traditional tensor decomposition methods. 
    }
\end{figure}

\subsection{Q5. Compression Speed}
We compare the total compression time of all considered methods.
We use the experimental setups in Section~\ref{sec:exp:tradeoff}.
For each dataset, we chose the setting of $R$ and $h$ with the smallest number of parameters. For the competitors, we used the settings with the fitness most similar to our method.
As shown in Figure~\ref{fig:time},
while \method takes less time than NeuKron, which sometimes takes more than 24 hours, it is still slower than the non-deep-learning-based methods. 

In Sections~\romanum{7} of \cite{appendix}, we provide additional experimental results, including hyperparameter sensitivity analysis.

\section{Conclusions} \label{sec:conclusion}
\begin{figure}[t]
    \centering
    \includegraphics[width=0.7\linewidth]{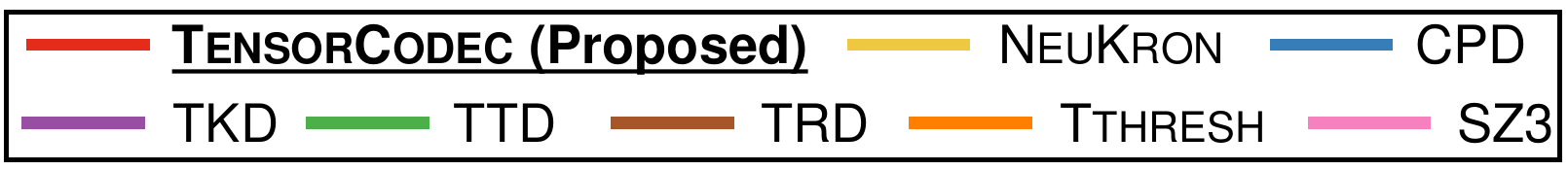}
    \vspace{-2mm}
    \includegraphics[width=0.9\linewidth]{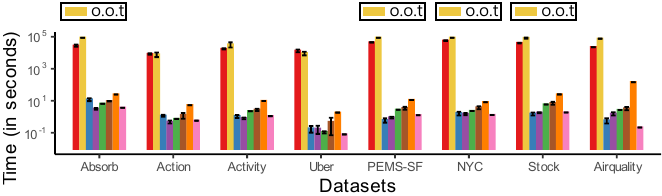}
\caption{The total compression time of \method and the competitors. 
NeuKron runs out of time on some datasets,  taking more than 24 hours.
}
\label{fig:time}
\end{figure}

In this work, we study the problem of lossy compression of tensors and devise \method that compactly compresses tensors without strong data assumptions.
Our main ideas are (a) Neural Tensor-Train Decomposition, (b) high-order tensor folding, and (c) mode-index reordering.
Using $8$ real-world tensors, we show the following advantages of \method:
\begin{itemize}[leftmargin=*]
    \item \textbf{Concise:} 
    It offers up to $7.38 \times$ more compact compression than the best competitor with similar reconstruction error. 
    \item \textbf{Accurate:} 
    When compression ratios are comparable, \method achieves up to $3.33 \times$ lower reconstruction error than the competitor with the smallest error.
    \item \textbf{Scalable:} 
    Empricially, its compression time is linear in the number of entries.
    It reconstructs each entry in logarithmic time with respect to the largest mode size. 
\end{itemize}
We plan to improve the speed of \method and theoretically analyze its approximation accuracy in future work.



{\small
\section*{Acknowledgements}
This work was funded by the Korea Meteorological Administration Research and Development Program ``Developing Intelligent Assistant Technology and Its Application for Weather Forecasting Process'' (KMA2021-00123).
This work was supported by Institute of Information \& Communications Technology Planning \& Evaluation (IITP) grant funded by the Korea government (MSIT)  (No. 2022-0-00157, Robust, Fair, Extensible Data-Centric Continual Learning) (No. 2019-0-00075, Artificial Intelligence Graduate School Program (KAIST)) (No.2021-0-02068, Artificial Intelligence Innovation Hub)
This work was supported by the National Research Foundation of Korea (NRF) grant funded by the Korea government (MSIT) (No. NRF-2020R1C1C1008296) (No. NRF-2021R1C1C1008526).
}

\bibliographystyle{IEEEtran}
\bibliography{bib}
\end{document}